\documentclass[letterpaper]{article} %
\usepackage[hyphens]{url}  %
\usepackage{graphicx} %
\usepackage{natbib}  %
\usepackage{caption} %
\usepackage[left=2.2cm, right=2.2cm, top=2.2cm, bottom=2.2cm]{geometry}
\usepackage{algorithm}
\usepackage[noend]{algorithmic}

\usepackage{tikz}

\usepackage{enumitem}
\setlist{topsep=2pt}

\usepackage{newfloat}
\usepackage{listings}
\DeclareCaptionStyle{ruled}{labelfont=normalfont,labelsep=colon,strut=off} %
\usepackage{amsmath}
\usepackage{amsthm}
\usepackage{amssymb}
\usepackage{mathtools}
\usepackage{pgf}

\newcommand{\BPP}{\mathsf{BPP}}

\newcommand{\PSPACE}{\mathsf{PSPACE}}

\newcommand{\Res}{\operatorname{Res}}

\newcommand{\Rset}{\mathbb{R}}

\theoremstyle{plain}
\newtheorem{theorem}{Theorem}

\newtheorem{lemma}[theorem]{Lemma}
\newtheorem{corollary}[theorem]{Corollary}
\theoremstyle{definition}
\newtheorem{definition}[theorem]{Definition}

\newtheorem{fact}[theorem]{Fact}
\newtheorem{observation}[theorem]{Observation}
\theoremstyle{remark}
\newtheorem{remark}[theorem]{Remark}
\newtheorem{example}[theorem]{Example}

\definecolor{bananamania}{rgb}{0.98, 0.91, 0.71}
\definecolor{bleudefrance}{rgb}{0.19, 0.55, 0.91}

\title{Identification for Tree-Shaped Structural Causal Models\\in Polynomial Time}
\author{
Aaryan Gupta\thanks{Department of Computer Science \& Engineering, Indian Institute of Technology Bombay, Mumbai, India.\\ Email: aaryanguptaxyz@gmail.com}
\and
Markus Bl\"a{ser}\thanks{Saarland University, Saarland Informatics Campus, Saarbr\"ucken, Germany.\\ Email: mblaeser@cs.uni-saarland.de}
}
    
\begin{document}
\date{}
\maketitle

\begin{abstract}
Linear structural causal models (SCMs) are used to express
and analyse the relationships between random variables.
Direct causal effects are represented as directed edges and confounding factors as bidirected edges. 
Identifying the causal parameters from correlations between the nodes is an open problem
in artificial intelligence. 
In this paper, we study SCMs whose directed component forms a tree. 
Van der Zander et al.~(AISTATS'22, PLMR 151, pp. 6770--6792, 2022) give a PSPACE-algorithm
for the identification problem in this case, which is a significant improvement over the 
general Gr\"obner basis approach, which has doubly-exponential time complexity in the number 
of structural parameters. In this work, we present a randomized polynomial-time algorithm, which
solves the identification problem for tree-shaped SCMs.
For every structural parameter, our algorithms
decides whether it is generically identifiable, 
generically 2-identifiable, or generically unidentifiable. (No other cases can occur.)
In the first two cases, it provides one or two  fractional affine square root terms
of polynomials (FASTPs) for the corresponding parameter, respectively.
\end{abstract}

\section{Introduction}

Linear structural causal models (SCMs) model
relationships between random variables \citep{bollen2014structural,duncan2014introduction}.
We are given random variables $V_0, V_1,\dots,V_n$, which depend linearly  
on  the other variables and an additional error term $\epsilon_i$, that is,
we can write $V_i=\sum_{j} \lambda_{j,i} V_j +\epsilon_i$.
The error terms are normally distributed 
with zero mean and covariances between these terms
given by some matrix $\Omega=(\omega_{i,j})$.
We consider only \emph{recursive models}, i.e., for all $j>i$, we have $\lambda_{j,i} = 0$. 
Such a model can be represented by a so-called mixed graph: The nodes $0,\dots,n$ correspond to the 
variables $V_0,\dots,V_n$. 
Directed edges represent a linear influence $\lambda_{j,i}$ of a parent node $j$ on its child
 $i$. Bidirected edges represent an additional correlation $\omega_{i,j}\neq 0$ between the error terms.
 
Figure~\ref{fig:SCM:1} shows a classical example of an SCM $M_1$.
A change of $V_0$ by $1$ implies a change of $\lambda_{0,1}$ of $V_1$ 
and a change of $\lambda_{01}\lambda_{12}$ of $V_2$. The covariance 
$\sigma_{01}$  between $V_0$ and $V_1$ 
is thus $\lambda_{01}$ and between $V_0$ and $V_2$, it is
$\lambda_{01}\lambda_{12}$. 

More general, we write 
the coefficients of all directed edges as a matrix 
$\Lambda=(\lambda_{i,j})$ and the coefficients of all bidirected edges as 
$\Omega=(\omega_{i,j})$. Then the matrix $\Sigma = (\sigma_{i,j})$ of covariances 
between the variables $V_i$ and $V_j$ is given by
\begin{equation}
\Sigma = (I - \Lambda)^{-1} \Omega (I-\Lambda)^{-T}, \label{eqn:SEM}
\end{equation}
see e.g.~\citet{drton2018algebraic}.
\begin{figure}[t]
\centering
\begin{minipage}{0.45 \columnwidth}
\centering
\begin{tikzpicture}[scale=0.9]
    \node[circle, draw] (z) at (0,0) {0};
    \node[circle, draw] (x) at (1.5,0) {1};
    \node[circle, draw] (y) at (3,0) {2};
    \draw [->] (z) -- node [midway,below] {$\lambda_{0,1}$} (x);
    \draw [->] (x) -- node [midway,below] {$\lambda_{1,2}$} (y);
    \draw [<->] (x) edge [bend left=45] node [midway,above]{$\omega_{1,2}$}(y);
  \end{tikzpicture}
\caption{SCM $M_1$, a classical IV example \label{fig:SCM:1}}
\end{minipage}~~~\begin{minipage}{0.45 \columnwidth}
\centering
\begin{tikzpicture}[scale=0.9]
    \node[circle, draw] (z) at (0,0) {0};
    \node[circle, draw] (x) at (1.5,0) {1};
    \node[circle, draw] (y) at (3,0) {2};
     \node[circle, draw] (a) at (0,1.5) {3};
    \node[circle, draw] (b) at (1.5,1.5) {4};
    \draw [->] (z) -- node [midway,below] {$\lambda_{0,1}$} (x);
    \draw [->] (x) -- node [midway,below] {$\lambda_{1,2}$} (y);
    \draw [->] (z) -- node [midway,left] {$\lambda_{0,3}$} (a);
    \draw [->] (x) -- node [midway,left] {$\lambda_{1,4}$} (b);
    \draw [<->] (b) edge [bend left=45] node [midway,above right]{$\omega_{2,4}$} (y);
    \draw [<->] (b) edge [bend left=45] node [midway, right]{$\omega_{1,4}$} (x); 
  \end{tikzpicture}
\caption{Tree-shaped SCM $M_2$ \label{fig:SCM:2}}
\end{minipage}
\end{figure}
The reverse problem, the identification of causal effects, is a major
open problem. Given the mixed graph and the matrix of covariances~$\Sigma$, 
calculate the matrix~$\Lambda$. (By (\ref{eqn:SEM}), this also
identifies $\Omega$.) 
We here deal with what is called the
\emph{identification problem}: Is the number of solutions
for a particular $\lambda_{i,j}$ finite and if yes,
provide corresponding symbolic expressions for $\lambda_{i,j}$.
If there is exactly one solution for $\lambda_{i,j}$, we call $\lambda_{i,j}$ identifiable,
if there are $k$ solutions, we call it $k$-identifiable, and if there is an infinite
number of solutions, we call it unidentifiable.

There has been a considerable amount of research on identification in linear SCMs
in economics, statistics, and artificial intelligence.
Of particular interest here is the pioneering work by \citet{Pearl2009}
on the computational aspects of identification and subsequent works.
Most approaches are based on \emph{instrumental variables} 
\citep{RePEc:cup:cbooks:9780521385824}.  
In Figure~\ref{fig:SCM:1},
one can calculate $\lambda_{1,2} = \frac{\lambda_{0,1}\lambda_{1,2}}{\lambda_{0,1}} = \frac{\sigma_{0,2}}{\sigma_{0,1}}$. $V_0$ is called an instrumental variable (IV) in this case. 
Since IVs are not complete, that is, not all identifiable parameters can be identified with IVs,
more complex generalizations of IVs have been developed,
for instance conditional instrumental variables \citep{RePEc:cup:cbooks:9780521385824,PearlConditionalIV,IJCAIInstruments},
instrumental sets \citep{BritoPearlUAI02,brito2010instrumental,brito2002graphical,instrumentalVariablesZander2016}, 
half-treks \citep{halftrek2012}, 
auxiliary instrumental variables \citep{chen2015incorporating},  
determinantal instrumental variables \citep{identifyingEdgeWiseDeterminantalDrton}, 
instrumental cutsets \citep{kumor2019instrumentalCutsets},  
or auxiliary cutsets \citep{kumor2020auxiliaryCutsets}.

A major drawback of all criteria mentioned above is that they are not complete,
that is, there are examples where they fail to identify a parameter that in principle 
is identifiable.  
An alternative approach is to solve the system of polynomial equations (\ref{eqn:SEM})
using Gr\"obner bases \citep{garcia2010identifying}. 
By the properties of Gr\"obner bases, this approach is complete, that is,
it identifies all parameters that are identifiable. 
However, Gr\"obner base algorithms typically have a doubly exponential running time and  
are often too slow to be used in practice once the SCMs get larger. 
 
\subsection{Our Results}

Since it seems to be difficult to find criteria that are complete and can be decided in
polynomial time, we restrict the underlying graph class. \Citet{DBLP:conf/aistats/ZanderWBL22} consider tree-shaped
SCMs $M = (V,D,B)$. Here, the graph $(V,D)$ of directed edges form a directed tree
with root $0$. Figure~\ref{fig:SCM:2} shows an example, but also Figure~\ref{fig:SCM:1} is one
with the tree being a path. \Citet{DBLP:conf/aistats/ZanderWBL22} propose an algorithm for tree-shaped SCMs called
TreeID, which works by algebraically solving systems of equations.
They prove that their algorithm is a PSPACE algorithm , 
a vast improvement over the general Gr\"obner basis approach,
but still with exponential running time in the worst case.
Moreover, they do not show that their algorithm is complete. 

In this work, we provide an algorithm which solves the identification problem in tree-shaped   
SCMs in randomized polynomial time. Hence the problem can be solved in the complexity
class $\BPP$.
For each variable $\lambda_{i,j}$, we can decide
in randomized polynomial time whether it is generically identifiable, $2$-identifiable, or 
unidentifiable. No other cases can occur. In the first two cases, we present one or two
rational expressions that might contain a square root (a so-called FASTP). 
Generically here means that the denominators are non-zero expressions but can become
zero when plugging in concrete values.
Generic identifiability is also called almost everywhere identifiability,
see e.g.\ \citet{garcia2010identifying,halftrek2012}.
 
Every missing bidirected edge in $M$ 
gives rise to a bilinear equation. Here a bidirected edge is called
missing if it does not appear in $M$, i.e. $e \notin B$.
\Citet{DBLP:conf/aistats/ZanderWBL22} show that if we have a cycle of missing edges,
then one can potentially identify or $2$-identify the parameters of the directed
edges entering the missing cycle. (Note that as the directed edges form a tree,
for every node $i$ in $V$, there is exactly one directed edge with head $i$.
We will often relate the parameter $\lambda_{p,i}$ of the unique edge entering $i$
to $i$ itself and say that $i$ is identifiable.)
However, one can be unlucky, and the missing cycle does not identify any parameter.
Therefore, \citet{DBLP:conf/aistats/ZanderWBL22} need to enumerate all cycles 
of missing edges, which can be exponentially many.
 
 As our first contribution, we present an algorithm which can find an identifying missing cycle in randomized polynomial
time, avoiding the need to enumerate all cycles. This is crucial to achieve polynomial running time. Second, we relate the concept of missing cycles to the theory of resultants. The crucial
point here is that all equations that arise from the missing edges have only two variables,
which is essentially the theory of plane algebraic curves, which is well understood.
This allows us to show the completeness of our algorithm for tree-shaped models.
Thus, tree-shaped SCMs are a class of graphs with an identification algorithm that is complete and has polynomial running time.
Furthermore, we associate the concept of rank with each missing edge, which essentially
means whether the corresponding bilinear polynomial factors or not. Rank-1 edges allow
to identify one parameter directly, but do not tell anything about the other parameter in
the corresponding equation. If we know one of the parameters in a rank-2 edge, however,
then this will identify the other parameter. But they do not identify one of the parameters
directly.

\section{Preliminaries}

We consider \emph{mixed graphs} $M = (V,D,B)$ with $n+1$ nodes $V = \{0,1,\dots,n\}$.
With each node we associate a random variable $V_i$.
$D$ is the set of directed edges. We denote directed edges either by $(i,j)$ or
$i \to j$. $B$ is the set of bidirected edges, which we denote by 
$\{i,j\}$ or $i \leftrightarrow j$.
We only consider the case when $M$ is acyclic, that is, $(V,D)$ is a directed acyclic
graph.
With each directed edge $i \to j$, we associate a variable $\lambda_{i,j}$ and with
every bidirected edge $u \leftrightarrow v$ a variable $\omega_{u,v}$.
Let $\Lambda$ and $\Omega$ be the corresponding $(n+1) \times (n+1)$-matrices.
Since $M$ is acyclic, we can always assume that $\Lambda$ is upper triangular.
The matrix $\Omega$ is symmetric. Let $\sigma_{i,j}$ be the covariance
between the random variables corresponding to nodes $i$ and $j$
and $\Sigma$ be the corresponding matrix. It is well known that
the $\sigma_{i,j}$ are polynomials in the parameters $\lambda_{i,j}$
and $\omega_{i,j}$, see (\ref{eqn:SEM}). 
(Because $\Lambda$ is upper triangular, the entries 
of $(I - \Lambda)^{-1}$ are indeed polynomials.)
Since the $\sigma_{i,j}$ are polynomials in the $\lambda_{i,j}$
and $\omega_{i,j}$, they are not independent variables
but might fulfill polynomial relations, which we will explore.

For a given mixed graph $M$, the \emph{identification problem}
asks to express the parameters $\lambda_{i,j}$ in terms
of the $\sigma_{i,j}$. We consider the generic version of the problem,
that is, the expressions that we provide for the $\lambda_{i,j}$ have
only to be defined on an open set. A parameter $\lambda_{i,j}$ is
generically identifiable if there is only one solution
for $\lambda_{i,j}$, $k$-identifiable, if there are $k$ different
solutions, and unidentifiable, if there are an infinite number of solutions.

An important tool is Wright's trek rule \citep{wright1921correlation, wright1934method}:
A \emph{trek} in $M$ between two nodes $i$ and $j$ is a path (with edges from $D$ and $B$)
such that no two arrow heads collide, that is, it is of the form
$i \leftarrow i_1 \leftarrow \dots \leftarrow u \leftrightarrow v \to \dots \to j_1 \to j$
or $i \leftarrow i_1 \leftarrow \dots \leftarrow u  \to \dots \to j_1 \to j$.
The part containing the $i$-nodes is the left part of the trek,
the part containing the $j$-nodes is the right part of the trek.
With each trek $\tau$ with associate a monomial $M(\tau)$,
which is the product of the edge labels. In the second case, 
we also multiply the monomial with $\omega_{u,u}$. Wright's rule says that
the covariances satisfy $\sigma_{i,j} = \sum_{\tau} M(\tau)$, where the sum is taken over
all treks between $i$ and $j$.

\newcommand{\miss}[1]{#1_\mathrm{miss}}

\subsection{Results by Van der Zander et al.}

For a mixed graph $M = (V,D,B)$, the \emph{missing edge graph} is the undirected
graph $\miss M = (V, \bar B)$. 
(Here, $\bar B$ denotes the complement of $B$.)
An edge in $\miss M$ is called a missing edge.
\citet{halftrek2012} show that the missing edges
are crucial for identification. More precisely, in the case of tree-shaped mixed
graphs with $0 \in V$ being the root,
\citet[Lemma 5]{DBLP:conf/aistats/ZanderWBL22} prove that a parameter $\lambda_{x,y}$
is generically identifiable ($k$-identifiable) iff the system of equations
\begin{align}
\lambda_{p,i} \lambda_{q,j} \sigma_{p,q} - \lambda_{p,i} \sigma_{p,j}
  - \lambda_{q,j} \sigma_{i,q} + \sigma_{i,j} = 0  \label{eq:zander1} \\
\lambda_{p,i} \sigma_{0,p} - \sigma_{0,i} = 0 \label{eq:zander2}   
\end{align}
has a unique solution for $\lambda_{x,y}$ ($k$ solutions, respectively). 
Above, there is an equation for 
each missing edge $i \leftrightarrow j$ not involving the root and each missing
edge $0 \leftrightarrow i$.
In each equation, $p$ denotes
the unique parent of $i$ and $q$ the unique parent of $j$ in the
tree $(V,D)$. We will use this naming convention frequently in the following.
If there is a missing edge $0 \leftrightarrow i$,
then $\lambda_{p,i}$ is generically identifiable as
$\lambda_{p,i} = \sigma_{0,i}/\sigma_{0,p}$.

\subsection{Polynomial Identity Testing}

We consider multivariate polynomials in indeterminates $x_1,\dots,x_n$ over some field. 
An \emph{arithmetic circuit} is an acyclic graph. The nodes with indegree $0$ (input gates) are labelled with variables
or constants from a given field. The inner nodes are either labelled with $+$ (addition gate)
or $*$ (multiplication gate). There is one unique node with outdegree $0$, the output gate.
An arithmetic circuit computes a polynomial at the output gate in the natural way,
see \citet{ramprasad} for more background.
Such arithmetic circuits occur in our setting for instance as iterated matrix products
or determinants with polynomials as entries. 

Given a polynomial computed by an arithmetic circuit, we want
to know whether it is identically zero. This problem is known as 
\emph{polynomial identity testing (PIT)}. Note that we cannot check this efficiently 
by simply computing
the polynomial at the output gate explicitly, since it can have an exponential number
of monomials. 
The Schwartz-Zippel lemma provides an efficient randomized solution for PIT.

\begin{lemma}[\citealt{pipSchwartz1980Zippel,pipZippel1979Schwartz}] \label{lem:sz}
Let $p(x_1,...,x_n)$ be a non-zero polynomial of total degree $\leq d$ over a field $K$. 
Let $S\subseteq K$ be a finite set and let $a_1,...,a_n \in S$ be selected  uniformly 
at random.
Then $\Pr(p(a_1,a_2,...,a_n)\not= 0))\geq 1- \frac{d}{|S|}$.
\end{lemma}

Instead of evaluating
the given arithmetic circuit explicitly, we first plug in random numbers for the variables.
When we then evaluate the circuit, there is no combinatorial explosion any more and
it can be done efficiently in polynomial time
in the degree $d$ and the circuit size $s$. Both will be polynomially bounded
in our algorithms.
It is a major open problem
in complexity theory whether the Schwartz-Zippel lemma can be derandomized.

The error probability
introduced by the Schwartz-Zippel lemma can be easily controlled. Assume that
the final algorithm calls the Schwartz-Zippel lemma $T$ times. Then we choose a set $S$
of size $d T / \epsilon$. By the union bound,
the overall error probability is $\le T \cdot \frac{d}{|S|} = \epsilon$.

Besides polynomials, we will also consider \emph{fractional affine square-root terms
of polynomials (FASTP)}, as defined by \citet{DBLP:conf/aistats/ZanderWBL22}. 
A FASTP is an expression of the form $\frac{p + q \sqrt{s}}{r + t \sqrt{s}}$,
where $p,q,r,t,s$ are polynomials given by arithmetic circuits. 
\Citet[Lemma 2]{DBLP:conf/aistats/ZanderWBL22} prove that we can test
whether a FASTP is the root of a quadratic polynomial with polynomials
as coefficients in randomized polynomial time by reduction to Lemma~\ref{lem:sz}.

\newcommand{\proofapp}{\hfill $\vartriangleright$}

\section{The Rank of Edges}

Let $\miss M = (V,\bar B)$ denote the graph of missing bidirected edges.
Let $i \leftrightarrow j \in \bar B$ and consider the corresponding
equation (\ref{eq:zander1}).
We write the coefficients in this equations as a $2 \times 2$-matrix 
$\left(\begin{smallmatrix} \sigma_{p,q} & \sigma_{i,q} \\ 
\sigma_{p,j} & \sigma_{i,j} \end{smallmatrix}\right)$.

\begin{definition}
The \emph{rank} of a missing bidirected edge is the rank of the corresponding
$2 \times 2$-matrix.
\end{definition} 

Recall that the  $\sigma_{x,y}$ are polynomials in the $\lambda_{x,y}$ and $\omega_{x,y}$.
Above, by rank, we mean the rank over the corresponding rational function field.

It turns out that rank-1 and rank-2 edges provide different kind of equations.
(We can ignore potential rank-0 edges, since they only give trivial equations.)
We will show in the subsequent sections
that rank-1 edges uniquely identify one of the two parameters, but need not
tell anything about the other. Rank-2 edges do not identify any of the parameters directly
but can be used to transfer known values for one of the parameters to the other parameter 
and vice versa.

\subsection{Rank-1 Edges}

If a missing edge $i \leftrightarrow j$
has rank 1, then the left-hand side of the corresponding equation~(\ref{eq:zander1})
factorizes into two linear terms. 

\begin{lemma} \label{lem:rank1}
$\det \left( \begin{smallmatrix} \sigma_{p,q} & \sigma_{i,q} \\ 
\sigma_{p,j} & \sigma_{i,j} \end{smallmatrix}\right) = 0$ if and only if
$\lambda_{p,i} \lambda_{q,j} \sigma_{p,q} - \lambda_{p,i} \sigma_{p,j}
  - \lambda_{q,j} \sigma_{i,q} + \sigma_{i,j}$ factorizes
  (over the rational function field).
\end{lemma}

\begin{proof} %
Let $\alpha$ and $\beta$ be such that $\alpha \sigma_{p,q} = \beta \sigma_{i,q}$
and $\alpha  \sigma_{p,j}  = \beta \sigma_{i,j}$. Since the determinant
is zero, at least one of $\alpha$ and $\beta$ can be chosen nonzero. 
Assume it is $\beta$, the other case is symmetric. We have:
\begin{align*}
\lambda_{p,i} \lambda_{q,j} \sigma_{p,q} - \lambda_{p,i} \sigma_{p,j}
  - \lambda_{q,j} \sigma_{i,q} + \sigma_{i,j}
 & \quad = \lambda_{p,i} (\lambda_{q,j} \sigma_{p,q} - \sigma_{p,j}) - \lambda_{q,j} \sigma_{i,q} + \sigma_{i,j} \\
 & \quad = (\lambda_{p,i} - \frac{\alpha}{\beta})  (\lambda_{q,j} \sigma_{p,q} - \sigma_{p,j}).
\end{align*}
The converse of the statement is also true, since any factorisation provides
a rank-1-decomposition of the matrix. That is,
when
\[
\lambda_{p,i} \lambda_{q,j} \sigma_{p,q} - \lambda_{p,i} \sigma_{p,j}
  - \lambda_{q,j} \sigma_{i,q} + \sigma_{i,j}
     = \sigma_{p,q}(\lambda_{p,i} - a )(\lambda_{q,j} -b)
\]
with rational functions $a$ and $b$, then
\[
\left( \begin{matrix} \sigma_{p,q} & \sigma_{i,q} \\ 
\sigma_{p,j} & \sigma_{i,j} \end{matrix}\right) =
\sigma_{p,q} \left(\begin{matrix} 1 \\ b \end{matrix}  \right)
\left(\begin{matrix} 1  & a \end{matrix}  \right).
\]
This completes the proof.
\end{proof}

\subsection{Rank-2 Edges}

If a missing edge has rank 2, then
$\lambda_{q,j} \sigma_{p,q} - \sigma_{p,j} \not= 0$ as a polynomial. If it 
were identically $0$,
then by plugging this into (\ref{eq:zander1}), we get that
$\lambda_{q,j} \sigma_{i,q} - \sigma_{i,j} = 0$, too. This is however
a contradiction, since it yields a linear dependence between the two columns of
$\left( \begin{smallmatrix} \sigma_{p,q} & \sigma_{i,q} \\ 
\sigma_{p,j} & \sigma_{i,j} \end{smallmatrix}\right)$. 
Thus,  we can write
   $\lambda_{p,i} = \frac{\lambda_{q,j} \sigma_{i,q} - \sigma_{i,j}}{\lambda_{q,j} \sigma_{p,q} - \sigma_{p,j}}$
in the rank-2 case.
The same is true when we exchange the roles of $\lambda_{p,i}$ and $\lambda_{q,j}$.
This shows the following:

\begin{lemma} \label{lemma:id:rank2}
Let $i \leftrightarrow j$ be a missing edge of rank two. Then $\lambda_{p,i}$ is generically identifiable
(2-identifiable, unidentifiable)
if and only if $\lambda_{q,j}$ is generically identifiable (2-identifiable, unidentifiable, respectively). \qed
\end{lemma}

\begin{remark} \label{remark:FASTP}
If $\lambda_{q,j}$ is determined by a rational expression, so is $\lambda_{p,i}$.
If $\lambda_{q,j}$ is determined by a FASTP, so is $\lambda_{p,i}$. 
Moreover, the square root term is the same for $\lambda_{p,i}$ and $\lambda_{q,j}$.
\end{remark}

\begin{proof} %
The first item is obvious. For the second item, assume that
$\lambda_{q,j} = \frac{u + v \sqrt{s}}{s + t \sqrt{s}}$. Then
\begin{align*}
   \lambda_{p,i} & = \frac{\frac{u + v \sqrt{s}}{r + t \sqrt{s}} \sigma_{i,q} - \sigma_{i,j}}{\frac{u + v \sqrt{s}}{r + t \sqrt{s}} \sigma_{p,q} - \sigma_{p,j}} \\
    & = \frac{(u \sigma_{i,q} - r \sigma_{i,j}) +  (v \sigma_{i,q} - t \sigma_{i,j}) \sqrt{s}}
             {(u \sigma_{p,q} - r \sigma_{p,j}) +  (v \sigma_{p,q} - t \sigma_{p,j}) \sqrt{s}}.
\end{align*}
This is a FASTP with the same square root term $\sqrt{s}$.
\end{proof}

\subsection{Computing the Rank of an Edge}

How do we check whether a matrix of the form $\left( \begin{smallmatrix} \sigma_{p,q} & \sigma_{i,q} \\ 
\sigma_{p,j} & \sigma_{i,j} \end{smallmatrix}\right)$ has rank 1 or 2? We can
compute the entries of this matrix by computing 
$(I - \Lambda)^{-1} \Omega (1 - \Lambda)^{-T}$ and then compute the determinant,
which will be a polynomial in the $\lambda_{s,t}$ and $\omega_{s,t}$. This can
be done using a polynomial-sized arithmetic circuit. However, we cannot compute this polynomial explicitly,
since intermediate results might have an exponential number of monomials.

Luckily, this is an instantiation of the polynomial identity testing problem,
and by Lemma~\ref{lem:sz},
it is sufficient to replace the indeterminates by random numbers chosen from a large enough set.
The resulting determinant can be computed in polynomial time, since the numbers are all small.
We get:

\begin{lemma} \label{lem:rank:edge}
Given a mixed graph $M = (V,B,D)$ and a missing edge $i \leftrightarrow j$, 
there is a randomized polynomial time algorithm that determines
the rank of $i \leftrightarrow j$. \qed
\end{lemma}

\section{Composition and Elimination}

\Citet{DBLP:conf/aistats/ZanderWBL22} show that each missing edge 
not involving the root $0$
gives 
a bilinear equation
in $\lambda_{p,i}$ and $\lambda_{q,j}$, see (\ref{eq:zander1}).
In the following three sections, we
study this problem in a more abstract setting:
We are given a set of indeterminates $X$ and a set of bilinear equations 
such that each equation contains exactly two indeterminates. For each
pair of variables, there is at most one such equations. Therefore,
we can view the problem as an undirected graph $G$ with node set $X$
and there is an edge connecting two nodes iff there is an equation
involving these two variables. We will assume that $G$ is connected,
since if not, we can treat each connected component separately.

Let 
\begin{align}
    axy - bx + cy - d & = 0, \label{eq:res:1}\\
    Ayz - By + Cz - D & = 0  \label{eq:res:2}
\end{align}
be two such bilinear equations, sharing the indeterminate $y$. If we want to eliminate $y$,
we can write (assuming all denominators are nonzero)
\[
    y = \frac{bx + d}{ax + c}, \quad z = \frac{By + D}{Ay + C}
\]
and obtain by plugging the first equation into the second
\begin{align} 
   z %
     & = \frac{(Bb + Da)x + (Bd + Dc)}{(Ab + Ca)x + (Ad + Cc)}. \label{eq:z:1}
\end{align}
The new coefficients are given by the entries of the matrix product
  $\left(\begin{smallmatrix}
     B & D \\
     A & C
  \end{smallmatrix} \right) 
  \left(\begin{smallmatrix}
     b & d \\
     a & c 
  \end{smallmatrix} \right) =
  \left(\begin{smallmatrix}
     Bb + Da & Bd + Dc \\
     Ab + Ca & Ad + Cc
  \end{smallmatrix} \right)$.
This is the well-known composition of M\"obius transforms, see e.g.\ \citet{moebius}.
The definition of a M\"obius transform requires that the $2 \times 2$-matrix
is invertible, that is, we have a rank-2 edge.

We can get a meaningful interpretation even in the case of rank-1 edges if 
we use the language of resultants \citep{resultant, CJO:ideals}, a standard tool for
solving polynomial equations. 
Write both equation as polynomial equations in the variable $y$ with coefficients in
$R[x, z]$:
\begin{align*}
    f(y) = (ax + c) y - (bx + d) & = 0, \\
    g(y) = (Az - B) y + (Cz - D) & = 0 .
\end{align*}
The resultant w.r.t.~$y$ of the two polynomials is
\begin{align*}
   \Res_y(f,g) & = \det \left(\begin{array}{cc}
                       - (bx + d) & ax + c \\
                           Cz - D & Az - B 
                      \end{array}\right) \\
      & = - (Ab + Ca) xz + (Bb + Da)x \\
      & \qquad - (Ad + Cc)z  + Bd + Dc. 
\end{align*}
The fact that $\Res_y(f,g) = 0$ can be rewritten as
\[
   (Ab + Ca) xz + (Ad + Cc)z = (Bb + Da)x + Bd + Dc,
\]
which is equivalent to (\ref{eq:z:1}).

When eliminating the variable $y$ in (\ref{eq:res:1}) and (\ref{eq:res:2})
we implicitly assume a ``direction'' in the process, the starting variable is $x$
and we want to express $z$ in terms of $x$.
We can also do the opposite. In this case, we write
\[
    y = \frac{Cz - D}{- Az + B}, \qquad x = \frac{cy - d}{- ay + b}
\]
and get
\begin{align*}
   x %
     & = \frac{(cC + dA)z - (cD + dB)}{-(aC + bA) z + (aD + bB)} .
\end{align*}
This is nothing but the product of the so-called adjoint matrices:
\[ \left(\begin{matrix}
     c & -d  \\
     -a & b 
  \end{matrix} \right) 
   \left(\begin{matrix}
     C & - D \\
     -A & B
  \end{matrix} \right) =
  \left(\begin{matrix}
     cC + dA & -(cD + dB) \\
     -(aC + bA) & aD + bB 
  \end{matrix} \right).
\]  
The corresponding equation is
\[
   -(aC + bA) zx - (cC + dA)z + (aD + bB) x + (cD + dB) = 0,
\]
which is equivalent to  (\ref{eq:z:1}), as expected. 
Of course, this is also equivalent to the fact that $\Res_y(g,f)$ vanishes,
as $\Res_y(g,f) = - \Res_y(f,g)$.

This suggest the following modelling as an edge labeled directed graph $\vec G$. The node set is again
the set of variables $X$. Assume there is an equation involving $x$ and $y$ like in (\ref{eq:res:1}).
We add a directed edge $(x,y)$ with weight $\left(\begin{smallmatrix} b & d \\ a & c \end{smallmatrix}\right)$
and the reverse edge $(y, x)$ with weight $\left(\begin{smallmatrix} c & -d \\ - a & b \end{smallmatrix}\right)$.
Call the resulting set of directed edges $E$. For an edge $e \in E$, let $M_e$ denote the
corresponding weight. The corresponding polynomial is called $f_e$. Note that the polynomial
is the same irrespective of the direction of the edges, but the weights are different, they
are the adjoint matrices of each other. Figure~\ref{fig:SCM:ex} shows an example
of such a graph $\vec G$.

\begin{figure}
\centering
\begin{tikzpicture}[scale=0.65]
    \node[circle, minimum size = 0.8cm, draw] (y) at (0,0) {$y$};
    \node[circle, minimum size = 0.8cm, draw] (x) at (5,3) {$x$};
    \node[circle, minimum size = 0.8cm, draw] (u) at (10,0) {$u$};
    \node[circle, minimum size = 0.8cm, draw] (z) at (5,-3) {$z$};
    \draw [->] (z) edge[bend left = 15] node [midway,left] {$M_{(z,x)} = \left(\begin{smallmatrix} 1 & 0 \\ 0 & 1 \\
\end{smallmatrix}\right)$} (x);
    \draw [->] (x) edge[bend left = 15] (z);
    \draw [->] (x) edge [bend right=17] node [midway, above, xshift = -0.9cm, yshift  = 0.1cm]{$M_{(x,y)} = \left(\begin{smallmatrix} 1 & 0 \\ 2 & 1 \\
\end{smallmatrix}\right)$}(y);
    \draw [->] (y) edge [bend right=10] (x);
    \draw [->] (x) edge [bend left=17] node [midway,above, xshift = 0.7cm, yshift = 0.1cm]{$M_{(x,u)} = \left(\begin{smallmatrix} 1 & 2 \\ 2 & 1 \\
\end{smallmatrix}\right)$}(u);
    \draw [->] (u) edge[bend left=10] (x);
    \draw [->] (y) edge [bend right=17] node [midway,below, xshift = -0.9cm, yshift = -0.1cm]{$M_{(y,z)} = \left(\begin{smallmatrix} 1 & 0 \\ -2 & 1 \\
\end{smallmatrix}\right)$}(z);
    \draw [->] (z)  edge[bend right=10] (y);
    \draw [->] (u) edge [bend left=17] node [midway,below, xshift = 0.7cm, yshift = -0.1cm]{$M_{(u,z)} = \left(\begin{smallmatrix} 1 & -1 \\ 0 & 1 \\
\end{smallmatrix}\right)$}(z);
    \draw [->] (z) edge[bend left=10] (u);
  \end{tikzpicture}
\caption{An example of the graph $\vec G$. For simplicity, we show only one weight
for each pair of edges between two nodes. The other edge has the corresponding adjoint matrix as weight.
\label{fig:SCM:ex}}
\end{figure}
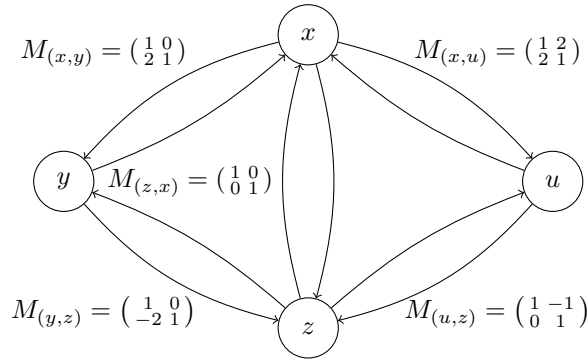

\begin{lemma} \label{lem:res}
Let $x_1,\dots,x_t$ be a simple directed path in $\vec G$. Then the coefficients of
\[
  \Res_{x_{t-1}}( \dots \Res_{x_3}(\Res_{x_2}(f_{(x_1,x_2)}, f_{(x_2,x_3)}), f_{(x_3,x_4)}),\dots, f_{(x_{t-1},x_t)}).
\]
are given by $M_{(x_{t-1},x_t)} \cdots M_{(x_2,x_3)} \cdot M_{(x_1,x_2)}$.
\end{lemma}

\begin{proof} %
This easily follows from the equivalence between repeated resultants
and matrix products proven before the lemma.
\end{proof}

\section{Identification Using Cycles}

We can identify a variable, if we close the path in Lemma~\ref{lem:res} by adding an edge $(x_t,x_1)$,
obtaining a (simple) cycle. This corresponds to an equation
\[ 
    x_1 = \frac{b x_1 + d}{ax_1 + c}  \qquad \text{or} \qquad
    a x_1^2 + (c - b) x_1 - d = 0
\]
where the coefficients are given by the matrix product
$M_{(x_t,x_1)} \cdot M_{(x_{t-1},x_t)} \cdots M_{(x_2,x_3)} \cdot M_{(x_1,x_2)}$. 

\begin{lemma} \label{lem:cycle:ident}
Let $\vec G$ be as above and let $x_1,\dots,x_t,x_1$ be a simple directed cycle in $\vec G$.
Let $\left(\begin{smallmatrix} b & d \\ a & c \end{smallmatrix}\right)$ be the product
$M_{(x_t,x_1)} \cdot M_{(x_{t-1},x_t)} \cdots M_{(x_2,x_3)} \cdot M_{(x_1,x_2)}$. 
\begin{enumerate}
\item If $a \not= 0$, then $x_1$ has at most two solutions (depending on the discriminant).
\item If $a = 0$ but $c - b \not= 0$, then $x_1$ has exactly one solution.
\item If $a = c- b = 0$ but $d \not= 0$, then there is no solution.
\item If $a = c-b = d = 0$, then $x_1$ has infinitely many solutions. 
\end{enumerate}
\end{lemma}

(In our application, case 3 cannot happen, since we assume that the covariances 
are in accordance with the model.)

\begin{proof} %
In the first case, depending whether the discriminant $\Delta := (c - b)^2 + 4ad$ is $>0$,
equals $0$ or is $< 0$, we have two, one, or zero solutions, respectively.
They are given by $\frac{ - (c-b) \pm \sqrt{\Delta}}{2a}$. 
In the second case, the solution is given by $\frac{d}{c-b}$.
The other two cases are clear.
\end{proof}

\begin{remark}
$a = c-b = d = 0$ in item 4 of Lemma~\ref{lem:cycle:ident} is equivalent to the statement that 
$\left(\begin{smallmatrix} b & d \\ a & c \end{smallmatrix}\right)$ is a multiple of the identity
matrix.
\end{remark}

\Citet{DBLP:conf/aistats/ZanderWBL22} also use cycles to eliminate variables. However,
they compute the coefficients not by a repeated matrix product. Instead, they compute
it via recursive determinants. To find a cycle that allows to identify a variable,
they simply enumerate all cycles, which gives an $\PSPACE$ upper bound on the complexity.
By computing with repeated matrix products, we will see
that we can efficiently find
a cycle that allows to identify a variable directly.

\begin{definition}
We call a cycle in Lemma~\ref{lem:cycle:ident} that satisfies the first or second condition
\emph{identifying}.
\end{definition}

\begin{example}
The cycle $x \to z \to y \to x$ in Figure~\ref{fig:SCM:ex} is not identifying,
since $M_{(z,x)} M_{(y,z)} M_{(x,y)} = 
\left(\begin{smallmatrix} 1 & 0 \\ 0 & 1 \end{smallmatrix}\right)$. 
On the other hand, the cycle
$x \to u \to z \to x$ is identifying, since 
$M_{(z,x)} M_{(u,z)} M_{(x,u)} = 
\left(\begin{smallmatrix} -1 & 1 \\ 2 & 1 \end{smallmatrix}\right)$
\end{example}

Now assume that the system of polynomial equations of $G$ has a finite number of solutions.
Then there is a subset of size $n$ of these equations, such that this system
has a finite number of solutions, where $n := |X|$ is the number of variables.
Graph-theoretically, the graph with these $n$ edges consists of connected
components such that in each connected component the number of nodes equals the number of edges in it.
(If there are fewer edges in one component, then we do not have a finite number of solutions.
If there are more edges in one component, then there is another one that has fewer edges,
because there are in total $n$ nodes and $n$ edges.)
A connected component with $c$ nodes and $c$ edges is an undirected tree with one additional
edge which creates a cycle. The system of the equations which belong to the edges
of the cycle has a finite number of solutions, too. (This follows easily,
since every variable not in the cycle appears in only one equation.)
It is well-known \citep{CJO:ideals} that such a system of polynomial equations has a finite
number of solutions if the iterated resultant of the polynomials 
(as in Lemma~\ref{lem:res}) is nonzero.
By the correspondence between resultants and matrix products (Lemma~\ref{lem:res})
we get:

\begin{lemma} \label{lem:correct:0}
If the system of equations given by $G$ has finitely many solutions, 
then $\vec G$ has an identifying cycle. \qed
\end{lemma}

We also have a converse in the case when all edges in $G$ are rank-2 edges. 

\begin{lemma} \label{lem:correct:1}
Let all edges in $G$ be rank-2 edges. If there is an identifying cycle in $\vec G$,
then the system of equations given by $G$ has a finite number of solutions.
\end{lemma}
\begin{proof}
Since the cycle is identifying, we get a nontrivial equation
for one of the variables, say $x_1$ (cases 1--3 in Lemma~\ref{lem:cycle:ident}).
Since all edges have rank-2, the corresponding transformations
in $\vec G$ are invertible M\"obius transforms. It is well-known
that the M\"obius transform is an isomorphism of the projective 
line, see \citet{moebius}. Therefore, each variable in $G$ (projectively)
has the same number of solutions, since we can transform them into
each other using the M\"obius transforms along the edges. (Recall that
$G$ is connected.)
\end{proof}

\begin{remark} \label{rem:numbersols}
Note that from the proof of Lemma~\ref{lem:correct:1} it even follows that 
the number of total solutions is the number of solutions for $x_1$.
\end{remark}

\section{Finding Identifying Cycles}

\label{sec:finding}

A \emph{directed walk} in a directed graph $\vec G = (X,E)$ 
is a sequence of nodes $x_0,\dots,x_t$ (potentially with repetitions)
such that $(x_i,x_{i+1})$ is an edge in $\vec G$ for all $0 \le i < t$. 
$t$ is the length of the walk.
A walk with no repetitions is called a \emph{path}.
A walk is \emph{closed} if $x_0 = x_t$.

Now assume that each edge $e$ in $\vec G$ has a weight $w_e$. Let $W$ be the weighted adjacency
matrix, that is, the entry in position $(i,j)$ is $w_{(v_i,v_j)}$ if $(v_i,v_j) \in E$ 
and $0$ otherwise.
(We assume 
w.l.o.g.\ that $X = \{v_1,\dots,v_n\}$.) For a walk $p$
consisting of nodes $x_1,\dots,x_t$, the weight $p$ is defined as 
$w(p) = w_{(x_1,x_2)} \cdots w_{(x_{t-1},x_t)}$, that is, the product of the edge weights. 
The weights can come from any ring. In the light of the previous section, we will choose the 
ring $\Rset[\sigma_{i,j}]^{2 \times 2}$, the ring of 
$2 \times 2$-matrices with polynomials in the $\sigma_{i,j}$ as entries,
which are itself polynomials in the $\lambda_{i,j}$ and $\omega_{i,j}$.
Note that this ring is not commutative, so the order in which
we multiply the weights matters.

\begin{fact} \label{fact:cycle:1}
The entry of $W^t$ in position $(i,j)$ is the sum of the weights of all directed walks of length $t$ 
from $v_i$ to $v_j$. \qed
\end{fact}

For a proof, see \citet{stanley}. If the edge weights for instance are all $1$, then
the entry of the matrix power $W^t$ in position $(i,j)$ just counts all walks of length $t$ from $v_i$ to $v_j$.
However, the weights of different paths might not simply add up as in this simple case,
but the weights of the all walks from $v_i$ to $v_j$ may add up to $0$. So the entry
of $W^t$ in position $(i,j)$ is zero, but there are walks from $v_i$ to $v_j$.
Therefore, we extend the weights to make every walk unique. Given the directed graph $\vec G = (X,E)$,
we construct a layered version $\vec G^{(t)} = (X^{(t)},E^{(t)})$ of it. For every node $v_i \in V$,
there are $t+1$ copies $v_i^{(0)},\dots,v_i^{(t)}$. And for each edge $(v_i,v_j) \in E$, there are $t$ copies
$(v_i^{(k)},v_j^{(k+1)})$, $0 \le k < t$. We call all nodes with index $k$ the $k$th layer of $\vec G^{(t)}$.
Edges only go from one layer to the next. One can think of the layers as being time steps.
In $\vec G^{(t)}$, every walk is neccessarily a path, since we cannot take an edge twice due
to the layered structure. Figure~\ref{fig:SCM:ex-walk} shows the layered
graph corresponding to the graph in Figure~\ref{fig:SCM:ex}.

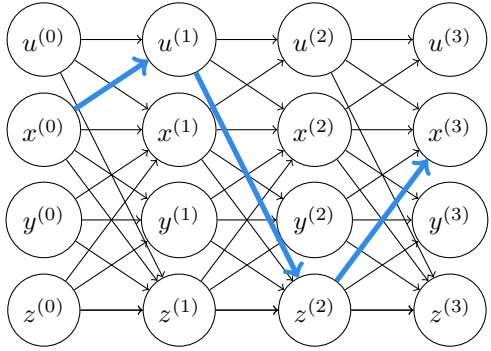
\begin{figure}
\centering
\begin{tikzpicture}[scale=0.6] 
     \node[circle, draw] (u) at (0,6) {$u^{(0)}$};
    \node[circle, draw] (x) at (0,4) {$x^{(0)}$};
    \node[circle, draw] (y) at (0,2) {$y^{(0)}$};
     \node[circle, draw] (z) at (0,0) {$z^{(0)}$};
     
      \node[circle, draw] (u2) at (3,6) {$u^{(1)}$};
     \node[circle, draw] (x2) at (3,4) {$x^{(1)}$};
    \node[circle, draw] (y2) at (3,2) {$y^{(1)}$};
      \node[circle, draw] (z2) at (3,0) {$z^{(1)}$};

     \node[circle, draw] (u3) at (6,6) {$u^{(2)}$};
    \node[circle, draw] (x3) at (6,4) {$x^{(2)}$};
    \node[circle, draw] (y3) at (6,2) {$y^{(2)}$};
     \node[circle, draw] (z3) at (6,0) {$z^{(2)}$};

     \node[circle, draw] (u4) at (9,6) {$u^{(3)}$};
    \node[circle, draw] (x4) at (9,4) {$x^{(3)}$};
    \node[circle, draw] (y4) at (9,2) {$y^{(3)}$};
     \node[circle, draw] (z4) at (9,0) {$z^{(3)}$};

      \draw [->] (u) edge node [midway,above]{}(u2);
      \draw [->] (u2) edge node [midway,above]{}(u3);
      \draw [->] (u3) edge node [midway,above]{}(u4);

      \draw [->] (x) edge node [midway,above]{}(x2);
      \draw [->] (x2) edge node [midway,above]{}(x3);
      \draw [->] (x3) edge node [midway,above]{}(x4);\

      \draw [->] (y) edge node [midway,above]{}(y2);
      \draw [->] (y2) edge node [midway,above]{}(y3);
      \draw [->] (y3) edge node [midway,above]{}(y4);

       \draw [->] (z) edge node [midway,above]{}(z2);
      \draw [->] (z2) edge node [midway,above]{}(z3);
      \draw [->] (z3) edge node [midway,above]{}(z4);

        \draw [->] (u) edge node [midway,above]{}(x2);
        \draw [->] (u) edge node [midway,above]{}(z2);

        \draw [->] (u2) edge node [midway,above]{}(x3);
      
        \draw [->] (u3) edge node [midway,above]{}(x4);
        \draw [->] (u3) edge node [midway,above]{}(z4);

        \draw [->, line width=2pt, color = bleudefrance] (x) edge node [midway,above]{}(u2);
        \draw [->] (x) edge node [midway,above]{}(y2);
        \draw [->] (x) edge node [midway,above]{}(z2);

         \draw [->] (x2) edge node [midway,above]{}(u3);
        \draw [->] (x2) edge node [midway,above]{}(y3);
        \draw [->] (x2) edge node [midway,above]{}(z3);

         \draw [->] (x3) edge node [midway,above]{}(u4);
        \draw [->] (x3) edge node [midway,above]{}(y4);
        \draw [->] (x3) edge node [midway,above]{}(z4);

        \draw [->] (y) edge node [midway,above]{}(x2);
        \draw [->] (y) edge node [midway,above]{}(z2);

         \draw [->] (y2) edge node [midway,above]{}(x3);
        \draw [->] (y2) edge node [midway,above]{}(z3);

         \draw [->] (y3) edge node [midway,above]{}(x4);
        \draw [->] (y3) edge node [midway,above]{}(z4);

         \draw [->] (z) edge node [midway,above]{}(x2);
         \draw [->] (z) edge node [midway,above]{}(y2);
         \draw [->] (z) edge node [midway,above]{}(z2);

         \draw [->] (z2) edge node [midway,above]{}(x3);
         \draw [->] (z2) edge node [midway,above]{}(y3);
         \draw [->] (z2) edge node [midway,above]{}(z3);

         \draw [->, line width=2pt, color = bleudefrance] (z3) edge node [midway,above]{}(x4);
         \draw [->] (z3) edge node [midway,above]{}(y4);
         \draw [->] (z3) edge node [midway,above]{}(z4);     

         \draw [->, line width=2pt, color = bleudefrance] (u2) edge node [midway,above]{}(z3);

  \end{tikzpicture}
\caption{The layered graph $\vec G^{(3)}$ corresponding to $\vec G$ from Figure \ref{fig:SCM:ex}. Weights are not drawn for simplicity.
The bold blue path in this graph corresponds to the identifying cycle 
$x \to u \to z \to x$. \label{fig:SCM:ex-walk}}
\end{figure}

\begin{observation} \label{obs:cycle:1}
There is a one to one correspondence between walks of length $t$ from $v_i$ to $v_j$ in $\vec G$
and paths from $v_i^{(0)}$ to $v_j^{(t)}$ in $\vec G^{(t)}$. \qed
\end{observation} 

If $(v_i,v_j)$ has weight $w_{(v_i,v_j)}$ in $\vec G$, then all edges $(v_i^{(k)},v_{j}^{(k+1)})$, $0 \le k < t$, have
weight $w_{(v_i,v_j)}$ in $\vec G^{(t)}$. For a walk $p$ in $\vec G$ of length $t$,
let $q$ be the corresponding path in $\vec G^{(t)}$. We have $w(p) = w(q)$, so the cancellation
problem is still there. We resolve this problem by extending the weight function of $\vec G^{(t)}$.
Let $x_{i,j}^{(k)}$, $1 \le i,j \le n$, $1 \le k \le t$, be new indeterminates. The new weight
of the edge $(v_i^{(k)},v_j^{(k+1)})$ is $\hat w(v_i^{(k)},v_j^{(k+1)}) = x_{i,j}^{(k+1)} w(v_i,v_j)$
So if the original weights are from some ring $R$, the new weights are now from the polynomial
ring $\hat R := R[x_{i,j}^{(k)}]$.

\begin{observation} \label{obs:cycle:2}
Let $p$ be a walk in $\vec G$ that visits the nodes $v_{i_0},\dots,v_{i_t}$ in $\vec G$ in this order.
Let $q$ be the corresponding path in $\vec G^{(t)}$. Then 
$\hat w(q) = \prod_{k = 1}^t x_{i_{k-1}, i_k}^{(k)} w(p)$.
In this way, each path $q$ with $w(p) \not=0$  gets a unique weight (in $\hat R$), since the monomial  
$\prod_{k = 1}^t x_{i_{k-1}, i_k}^{(k)}$ exactly describes the edges taken in $q$
and their order. \qed
\end{observation}

\begin{example}
If we identify the nodes $u,x,y,z$ with $v_1,v_2,v_3,v_4$ in 
Figure~\ref{fig:SCM:ex-walk}, then the bold blue path has weight 
$x_{2,1}^{(1)} x_{1,4}^{(2)} x_{4,2}^{(3)} \cdot
\left(\begin{smallmatrix} -1 & 1 \\ 2 & 1 \end{smallmatrix}\right)$.
$\left(\begin{smallmatrix} -1 & 1 \\ 2 & 1 \end{smallmatrix}\right)$
is the original weight in $\vec G$ and $x_{2,1}^{(1)} x_{1,4}^{(2)} x_{4,2}^{(3)}$
is the unique label of the path.
\end{example}

Let $\hat W$ be the weighted adjacency matrix of $\vec G^{(t)}$ using the weights $\hat w$.
By Fact~\ref{fact:cycle:1} and Observation~\ref{obs:cycle:2} we have:

\begin{lemma} \label{lem:cycle:1}
The entry of $\hat W^t$ corresponding to the nodes $v_i^{(0)}$ and $v_j^{(t)}$ 
is the sum of all $\prod_{k = 1}^t x_{i_{k-1}, i_k}^{(k)} w(p)$
over all walks $p$ of length $t$ from $v_i$ to $v_j$ in $\vec G$ 
(interpreted as a path in $\vec G^{(t)}$ using Observation~\ref{obs:cycle:1}). \qed
\end{lemma}

The matrix $\hat W$ has size $(t+1)n$. However, due to the layered structure of $G^{(t)}$,
it has a special structure: 
It has a block structure with $\hat W_1,\dots,\hat W_t$
being the only nonzero $n \times n$ blocks above the main diagonal,
\begin{equation} \label{eq:W}
    \hat W = \left(\begin{array}{ccccc} 
             0 & \hat W_1 & 0 & \dots & 0 \\
             0 & 0 & \hat W_2 & \dots & 0  \\
             \vdots & & \ddots & \ddots & \vdots \\
             0  & 0 & \dots & 0 & \hat W_t \\
             0  & 0 & \dots & 0 & 0 \\   
             \end{array} \right).
\end{equation}
Each block $\hat W_k$ contains the weights of the edges between layer $k-1$ and $k$.
All $\hat W_k$ have similar entries, as edges between the layers
are identical and the weights only differ in the index 
of the indeterminates $x_{i,j}^{(k)}$.
see (\ref{eq:W}) in the appendix for a visualization.
It is easy to check that $\hat W^t$ has only non-zero block, which is in the top right
and contains the matrix $\hat W_1 \dots \hat W_t$. Therefore, instead
of raising $\hat W$ to the $t$th power, we can simply compute the iterated product
$\hat W_1 \dots \hat W_t$, which is more efficient.

\subsection{Polynomial Identity Testing}

Note that we cannot compute the iterated product
$\hat W_1 \dots \hat W_t$ directly, since the entries are multivariate polynomials,
which can have an exponential number of monomials.
However, we do not need to compute the polynomials explicitly, we only need to check
whether they are nonzero. This problem is again an instance of PIT.

To check whether some entry in the iterated product
$\hat W_1 \dots \hat W_t$ is nonzero, we just need to replace the variables 
by random values. 
In our case, the weights are $2 \times 2$-matrices.
Note that in Lemma~\ref{lem:res}, the weights are multiplied from the left when proceeding.
This is not a problem, since we can simply transpose all $2 \times 2$-matrices
and use the fact that $(AB)^T = B^T A^T$. 

\begin{lemma} \label{lem:walk:pit}
There is a randomized polynomial time algorithm which given a node $v_i$ checks
whether $v_i$ lies on an identifying closed walk of length at most $t$.
\end{lemma}

\begin{proof}%
Let $\hat P$ be the entry in position $(i,i)$
of the product $\hat W_1 \dots \hat W_t$. 
It is a polynomial in the variables $x_{i,j}^{(k)}$
whose coefficients are $2 \times 2$-matrices with polynomials
in the $\sigma_{i,j}$ as entries. We can interpret
$\hat P$ as a $2 \times 2$-matrix whose entries are polynomials
in the two sets of variables $x_{i,j}^{(k)}$ and $\sigma_{i,j}$.
We have that 
\[ 
  \hat P = \sum_{p} \prod_{k = 1}^t x_{i_{k-1}, i_k}^{(k)} w(p)
  \]
where the sum is over all closed 
walks $p$ of length $t$ from $v_i$ to $v_i$ in $\vec G$ by Lemma~\ref{lem:cycle:1}.
Thus there is an identifying cycle in $\vec G$ iff $\hat P$ is not a multiple
of the identity matrix by Lemma~\ref{lem:cycle:ident}, that is 
$\hat P_{2,1} \not= 0$ or $\hat P_{11} - \hat P_{22} \not= 0$
or $\hat P_{1,2} \not=0$. This is an instance of PIT and therefore
can be solved in randomized polynomial time using the Schwartz-Zippel Lemma.
We plug in random values for the $x_{i,j}^{(k)}$ and replace
the $\sigma_{i,j}$ by the polynomial expressions in the $\lambda_{i,j}$
and $\omega_{i,j}$ given by (\ref{eqn:SEM}) and plug in random values for these variables, too.
Then we simply evaluate.
\end{proof}

\subsection{Selfreducibility}

Note that while we now have an algorithm to test whether there is an identifying closed 
walk, this algorithm does not construct the identifying closed walk.
Furthermore, an identifying closed walk is not necessarily a simple
cycle. However, our interpretation in terms of resultants in Lemma~\ref{lem:res}
is only valid for simple cycles.

The first problem can be solved using a technique called selfreducibility.
For instance, assume we have an algorithm that decides the well-known
satisfiability problem, how can we find a satisfying assignment?
We set the first variable to $1$ and check whether the resulting
formula is still satisfiable. If yes, we know that we can set $x_1$
to $1$. Otherwise we set $x_1$ to $0$. Then we go one with $x_2$ and so forth.

We can do the same here: We remove an arbitrary edge from the graph,
check whether there is still an identifying closed walk. If yes, we can safely
remove the edge. If not, we know that the edge is in every remaining 
identifying closed walk and we keep it. The edges that remain will form  an identifying closed
walk.

The second problem will be solved by the following lemma:

\begin{lemma} \label{lem:shortest} 
The shortest identifying closed walk in $G$ is always a cycle. %
\end{lemma}

\begin{proof}%
Let $p$ be a closed walk that is identifying, that is, $w(p)$ is not a multiple of the
identity matrix. Assume $p$ is not simple.
We walk along $p$ until we visit a node $v$ twice. The 
part of $p$ between these two occurrences does not contain any node
twice by construction, hence it is a (simple) cycle. Call it $c$.
If $w(c)$ is not a multiple of the identity matrix, then $c$ is a shorter
identifying closed walk. If $w(c)$ is a multiple of the identity matrix, say 
$\alpha \cdot I$,
then we can remove $c$ from $p$. Let $p'$ be the resulting closed walk.
Since the weights are iterated matrix products, we have
$\alpha w(p') = w(p)$, hence $w(p')$ is not a multiple of the identity matrix.
(Note that $\alpha \not= 0$, since $w(p) \not= 0$.)
Again we found a shorter identifying closed walk.
\end{proof}

\subsection{An Algorithm for Finding Identifying Cycles}

Algorithm~\ref{alg:idcycle} summarizes the results of this section.

\begin{lemma} \label{lem:idcycle}
Algorithm~\ref{alg:idcycle} finds an identifying cycle if there is one.
It runs in randomized polynomial time. %
\end{lemma}

\begin{proof}%
By increasing the length $t$, we make sure that if we find an
identifying closed walk, it will be a shortest one. By Lemma~\ref{lem:shortest},
this will be an identifying cycle. If there is an identifying walk,
then the shortest has length at most $n$. Therefore, the algorithm
finds a identifying cycle if there is one.

The algorithm is randomized due to the use of the Schwartz-Zippel lemma.
However, the error can be made small by using large enough numbers
as explained before. The running time is clearly polynomial.
\end{proof}

\begin{algorithm}[tb]
\caption{Finding identifying cycles}
\label{alg:idcycle}
\textbf{Input}: A graph $\vec G = (X,E)$ \\
\textbf{Output}: An identifying cycle
\begin{algorithmic}[1] %
\FOR{$t = 2,\dots, n$}
\STATE Build the corresponding graph $\vec G^{(t)}$ and matrix $\hat W$.
\STATE Replace all indeterminates by random values according to the Schwartz-Zippel lemma.
\STATE Compute $\hat W_1 \cdots \hat W_t$.
\STATE Check whether there is an identifying closed walk
using Lemma~\ref{lem:walk:pit} for $i = 1,\dots,n$.
\IF{ there is such a walk}
\STATE Compute such a walk $p$ using selfreducibility.
\STATE \textbf{return} $p$
\ENDIF
\ENDFOR
\STATE \textbf{return} no solution
\end{algorithmic}
\end{algorithm}

\section{Identification Using Rank-1 Edges}

We consider a missing edge $i \leftrightarrow j$ of rank 1. Again, $p$ is the parent of $i$
and $q$ of $j$ in the directed graph $(V,D)$. We will show that a rank-1 edge always uniquely
identifies one of the nodes (and need not provide any information
about the other node, since when we plug in the unique value in the equation,
the equation will be identically $0$). 

\begin{definition}
Let $U,W$ be two sets of nodes. A pair $(S,T)$ of sets of nodes \emph{trek-separates}
$U$ and $W$ if every trek from any $i \in U$ to any $j \in W$ intersects $S$ on its left side
or $T$ on its right side.
\end{definition}

\begin{theorem}[see {\citealp[Theorem 11.1]{drton2018algebraic}}]
The submatrix with rows corresponding to $U$ and columns corresponding to $W$ has 
(generic) rank $\le r$ if there is a pair $(S,T)$ with $|S| + |T| \le r$ such that
$(S,T)$ trek-separates $(U,W)$.
\end{theorem}

In our case $U = \{i,p\}$ and $W = \{j,q\}$ ($p$ and $q$ can be potentially the same).
We are interested when
$\left(\begin{smallmatrix} \sigma_{p,q} & \sigma_{i,q} \\ 
\sigma_{p,j} & \sigma_{i,j} \end{smallmatrix} \right)$
has rank $1$. In this case either $S$ or $T$ consists of one node and the other set is empty.
The situation is symmetric, so we assume that $|T| = 1$.
$T$ cannot be $\{j\}$, since then the trek from $i$ to $q$ (consisting only of directed edges)
is not separated. Let $0,p_1,\dots,p_t,p,i$ be the directed path from the root $0$ to $i$.
The trek consisting of the left side being any suffix of this path starting in node
$x \in \{0,p_1,\dots,p_t,p,i\}$ and the bidirected edge $x \leftrightarrow j$ 
is not trek-separated by $(S,T)$. Therefore, all bidirected edges $x \leftrightarrow j$
are missing! See Figure~\ref{fig:rank-1:1} for an illustration.

We will now prove that all these missing edges are useless for identifying any of the 
$\lambda$-parameter in the directed path.
\begin{itemize}
\item As $0 \leftrightarrow j$ is missing, we can identify 
$\lambda_{q,j} = \frac{\sigma_{0,j}}{\sigma_{0,q}}$.
\item Since $p_1 \leftrightarrow j$ is missing, we get by (\ref{eq:zander1})
\[
  \lambda_{0,p_1} \lambda_{q,j} \sigma_{0,q} - \lambda_{0,p_1} \sigma_{0,j}
  - \lambda_{q,j} \sigma_{p_1,q} + \sigma_{p_1,j} = 0.
\]
Plugging in $\lambda_{q,j} = \frac{\sigma_{0,j}}{\sigma_{0,q}}$, we get
   $\sigma_{0,j} \sigma_{p_1,q} - \sigma_{0,q} \sigma_{p_1,j} = 0$.
Thus this edges does not yield any information about $\lambda_{0,p_1}$ (and
we already identified $\lambda_{q,j}$). Furthermore, the equation
means that the missing edge $p_1 \leftrightarrow j$ is rank-1, too,
and we get that $\lambda_{q,j} = \frac{\sigma_{0,j}}{\sigma_{0,q}} 
= \frac{\sigma_{p_1,j}}{\sigma_{p_1,q}}$.
\item Since $p_2 \leftrightarrow j$ is missing, we get 
\[
  \lambda_{p_1,p_2} \lambda_{q,j} \sigma_{p_1,q} - \lambda_{p_1,p_2} \sigma_{p_1,j}
  - \lambda_{q,j} \sigma_{p_2,q} + \sigma_{p_2,j} = 0.
\]
Plugging in $\lambda_{q,j} = \frac{\sigma_{p_1,j}}{\sigma_{p_1,q}}$ yields
   $\sigma_{p_1,j} \sigma_{p_2,q} - \sigma_{p_1,q} \sigma_{p_2,j} = 0$.
Thus this edge does not yield any information about $\lambda_{0,p_2}$ (and
we already identified $\lambda_{q,j}$). Moreover, the equation
means that the missing edge $p_2 \leftrightarrow j$ is rank-1, too,
and we get that $\lambda_{q,j} = \frac{\sigma_{0,j}}{\sigma_{0,q}}
= \frac{\sigma_{p_1,j}}{\sigma_{p_1,q}} 
= \frac{\sigma_{p_2,j}}{\sigma_{p_2,q}}$.
\item We can go on inductively until we reach the node $p$ and
get $\lambda_{q,j} = \frac{\sigma_{0,j}}{\sigma_{0,q}} = \dots 
=\frac{\sigma_{p,j}}{\sigma_{p,q}}$.
\item Finally, since the edge $i \leftrightarrow j$ is missing,  
\[
  \lambda_{p,i} \lambda_{q,j} \sigma_{p,q} - \lambda_{p,i} \sigma_{p,j}
  - \lambda_{q,j} \sigma_{i,q} + \sigma_{i,j} = 0.
\]
Plugging in $\lambda_{q,j} = \frac{\sigma_{p,j}}{\sigma_{p,q}}$, we get
   $\sigma_{p,j} \sigma_{i,q} - \sigma_{p,q} \sigma_{i,j} = 0$
and the equation does not depend on $\lambda_{p,i}$.
\end{itemize}

There is the special case that $p = q$. This case is treated in the exact same way:
$q$ lies on any directed path to $i$. However, this does not matter
in the definition of trek-separation, since the set $T = \{q\}$ can only intersect
the right side of a trek. (Same when the set $S$ is nonempty.)

Finally $i$ could be a predecessor of $j$ (or vice versa). 
Assume that $i$ is a predecessor of $j$ (and of $q$ in this case, 
since the underlying structure is a tree.)
$i$ could be equal to $q$. We need to intersect the right side of a given trek,
because the trek $i \to \dots \to j$ has no left side. $T = \{j\}$ is not
possible, since this does not intersect the trek $i \to \dots \to q$.
Thus we need to choose $T = \{q\}$ (or any other node between $i$ and $q$). 
Let $0,p_1,\dots,p_t,p$ be the nodes on the path from the root $0$ to $p$.
All the edges $0 \leftrightarrow j$, $p_1 \leftrightarrow j$,\dots,
$p_t \leftrightarrow j$, and $p \leftrightarrow j$
have to be missing, because otherwise, we do not intersect the trek
$i \leftarrow p \leftarrow \dots \leftarrow 0 \leftrightarrow j$ and the corresponding
ones with $0$ replaced by $p_1,\dots,p_t$ or $p$.
See Figure~\ref{fig:rank-1:2} for an illustration.
Thus the missing edges we get are the same as in the first case,
and the same calculations work.

The case when $j$ is a predecessor of $i$ is symmetric. In this case, we 
get similar equations and can identify $i$.

Altogether, we have proven the following

\begin{lemma} \label{lem:missing:rank-1}
For every  missing rank-1 edge $i \leftrightarrow j$, the edges $0 \leftrightarrow i$ 
or $0 \leftrightarrow j$ (or both) are missing.
The corresponding expressions satisfy the equation of the edge $i \leftrightarrow j$.
\qed
\end{lemma}

\begin{corollary}
Let $C$ be a connected component of the missing edge graph. If $C$ contains a rank-1 edge,
then $C$ also contains the root. \qed
\end{corollary}

\begin{figure}
\centering
\scalebox{0.7}{
\begin{tikzpicture}[rotate=-90]
    \node[circle, draw] (i) at (0,0) {$i$};
    \node[circle, draw] (j) at (2,0) {$j$};
    \node[circle, draw] (p) at (0,-1.5) {$p$};
    \node[circle, draw, fill = bananamania] (q) at (2,-1.5) {$q$};
    \node[circle, draw] (pt) at (0,-3) {$p_t$};
    \node (qt) at (2,-3) {};
    \node (mid) at (1,-5) {};
    \node[circle, draw] (p2) at (1,-7) {$p_2$};
    \node[circle, draw] (p1) at (1,-9) {$p_1$};
    \node[circle, draw] (root) at (1,-11) {$0$};
    \draw [->] (p) -- node [midway,above] {$\lambda_{p,i}$} (i);
    \draw [->] (q) -- node [midway,above] {$\lambda_{q,j}$} (j);
    \draw [->] (pt) -- node [midway,above] {$\lambda_{p_t,p}$} (p);
    \draw [->, dashed] (mid.center) -- (pt);
    \draw [->, dashed] (qt.center) -- (q);
    \draw [dashed] (mid.center) -- (qt.center);
    \draw [dashed] (p2) -- (mid.center);
    \draw [->] (p1) -- node [midway,above] {$\lambda_{p_!,p_2}$} (p2);
    \draw [->] (root) -- node [midway,above] {$\lambda_{0,p_1}$} (p1);
    \draw [<->, color = bleudefrance] (i) -- (j);
    \draw [<->, color = bleudefrance] (p) -- (j);
    \draw [<->, color = bleudefrance] (pt) -- (j);
    \draw [<->, color = bleudefrance] (p2) edge[bend right=27]  (j);
    \draw [<->, color = bleudefrance] (p1) edge[bend right=27]  (j);
    \draw [<->, color = bleudefrance] (root) edge[bend right=27]  (j);
\end{tikzpicture}
}
\caption{If the missing edge $i \leftrightarrow j$ has rank 1, then
all the other blue edges have to be missing by trek-separability.
The best choice for $T$ is $T = \{q\}$ (drawn yellow). \label{fig:rank-1:1}}
\end{figure}
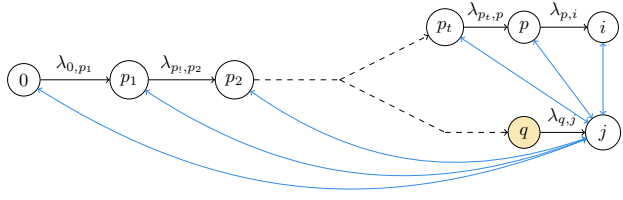

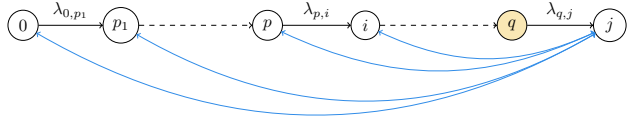
\begin{figure}
\centering
\scalebox{0.65}{
\begin{tikzpicture}
    \node[circle, draw] (i) at (7,0) {$i$};
    \node[circle, draw] (j) at (12,0) {$j$};
    \node[circle, draw] (p) at (5,0) {$p$};
    \node[circle, draw, fill = bananamania] (q) at (10,0) {$q$};
    \node[circle, draw] (p1) at (2,0) {$p_1$};
    \node[circle, draw] (root) at (0,0) {$0$};
    \draw [->] (p) -- node [midway,above] {$\lambda_{p,i}$} (i);
    \draw [->] (q) -- node [midway,above] {$\lambda_{q,j}$} (j);
    \draw [->, dashed] (p1) -- (p);
    \draw [->, dashed] (i) -- (q);
    \draw [->] (root) -- node [midway,above] {$\lambda_{0,p_1}$} (p1);
    \draw [<->, color = bleudefrance]  (i) edge[bend right = 25] (j);
    \draw [<->, color = bleudefrance] (p) edge[bend right= 25] (j);
    \draw [<->, color = bleudefrance] (p1) edge[bend right=30]  (j);
    \draw [<->, color = bleudefrance] (root) edge[bend right=30]  (j);
\end{tikzpicture}
}
\caption{If the missing edge $i \leftrightarrow j$ has rank 1, then
all the other blue edges have to be missing by trek-separability.
The set $T = \{q\}$ is drawn yellow. \label{fig:rank-1:2}}
\end{figure}

\begin{algorithm}[tb]
\caption{Identification}
\label{alg:final}
\textbf{Input}: A tree-shaped mixed graph $M = (V,D,B)$\\
\textbf{Output}: For each $\lambda_{p,i}$, we output whether it is generically identifiable,
$2$-identifiable, or unidentifiable. In the first two cases, we 
output corresponding FASTPs.
\begin{algorithmic}[1] %
\STATE Find all rank-1 edges in the missing edge graph $(V, \bar B)$ 
(using Lemma~\ref{lem:rank:edge}).
\STATE For each missing rank-1 edge $i \leftrightarrow j$,
check which of the parameters $\lambda_{p,i}$ or $\lambda_{q,j}$ we can identify 
(using Lemma~\ref{lem:missing:rank-1}).
Mark the  node $i$ or $j$, respectively, and compute
a corresponding rational expression.
\STATE Remove all rank-1 edges from the missing edge graph.
Call the resulting graph $H$.
Let $C_1,\dots,C_t$ be the connected components of $H$.
\FOR{each connected component $C_i$}
\IF  {$C_i$ contains a marked node}
\STATE Propagate the result to all
unidentified nodes in $C_i$ and produce
corresponding rational expressions.
\ELSE
\STATE Find an identifying cycle in $C_i$ (using Algorithm~\ref{alg:idcycle}).
\STATE If no such cycle is found, report that all nodes of $C_i$ are unidentifiable.
\STATE If the cycle produces one solution, then progagate it to all the nodes
of $C_i$ and compute corresponding rational expressions.
\STATE If the cycle produces two solutions, then propagate it to all the nodes 
of $C_i$ and compute corresponding FASTPs. 
\STATE Plug the FASTPs into the equations of $C_i$ and use 
PIT to check whether all equations are satisfied. 
If yes, keep the solution, otherwise, drop it.
\ENDIF
\ENDFOR
\end{algorithmic}
\end{algorithm}

\section{The Final Algorithm}

Algorithm~\ref{alg:final} solves the identification problem. For each 
$\lambda_{p,i}$ is decides in randomized polynomial time
whether $\lambda_{p,i}$ is generically identifiable, $2$-identifiable, or unidentifiable.
(Other cases cannot occur.) In the first two cases, it provides
corresponding FASTPs.

\begin{theorem}
Algorithm~\ref{alg:final} is correct. 
\end{theorem}
 
\begin{proof}
 Consider a variable $\lambda_{p,i}$. First assume that $i$ is contained
 in a rank-1 missing edge $i \leftrightarrow j$. If $i$ is connected
 to the root in $\miss M$, then $i$ is identified by (\ref{eq:zander2}). If $i$ is not
 connected to the root, then $j$ is. The solution for $j$
 satisfies the equation by Lemma~\ref{lem:missing:rank-1}. Therefore, it will
 not be useful for identifying $\lambda_{p,i}$. Thus we can safely remove
 all rank-1 edges in step 3. Assume that $\lambda_{p,i}$ is not 
 identified by a rank-1 edge.

If the rank-2 component of $\lambda_{p,i}$ contains a marked node, which is identifiable,
then each node in this component is identifiable by
Lemma~\ref{lemma:id:rank2}. And by the subsequent Remark~\ref{remark:FASTP},
we also get a rational expression.

If there is no marked node in the component, then the nodes in the component
are $k$-identifiable iff there is an identifying cycle by Lemmas~\ref{lem:correct:0}
and~\ref{lem:correct:1}.
If such a cycle is found, it can have either one or two solutions
by Lemma~\ref{lem:cycle:ident}. These one or two solutions can be propagated
using Lemma~\ref{lemma:id:rank2} and Remark~\ref{remark:FASTP}.
Note that we get in total at most two solutions for the component,
given by FASTPs (see Remark~\ref{rem:numbersols}).
If there is only one solution, then we are done. If there are two solutions,
we plug them into all equations of the component and see whether they both
satisfy all equations or just one of them. In the second case,
$\lambda_{p,i}$ is identifiable. In the first case, it is $2$-identifiable,
but not identifiable. (Note that we do not have to check equations of other
components, since they do not contain any variables of the present component).
We can check whether two FASTPs satisfy a bilinear equation
using the test by \citet[Lemma 2]{DBLP:conf/aistats/ZanderWBL22}, since
both FASTPs have the same square root term by Remark~\ref{remark:FASTP}.

If there is no identifying cycle, then the nodes are unidentifiable by 
Lemma~\ref{lem:correct:0}. Thus the algorithm is correct.
\end{proof}

The running time of Algorithm~\ref{alg:final} is clearly polynomial.
More concretely, it can be estimated as follows: We start with Algorithm \ref{alg:idcycle}. 
$n$ denotes the number of nodes of the SCM.
Instead of applying the Schwartz-Zippel lemma
in the end, it is more efficient to replace all indeterminates with random values right
from the start. 
The $\sigma_{i,j}$ can be computed in time $O(n^3 \log n)$ by repeated
squaring and
the entries of $\hat W_1 \dots \hat W_t$ for each $t = 1,\dots,n$ in total time $O(n^4)$
(using the standard $O(n^3)$ algorithm for matrix multiplication). 
Finding an identifying cycle then
takes $O(e)$ self reduction steps, where $e$ is the number of missing edges.
Each self reduction step costs $O(n^4)$. Thus the overall running 
time is $O(e \cdot n^4)$. (It can be improved to time 
$O(n^4 + e \cdot n^3)$ by using some dynamic update techniques.)
This time also dominates the running time of Algorithm \ref{alg:final}.
The worst case is when we have only one connected component. Finding an identifying
cycle takes time $O(e \cdot n^4)$ and checking the equations in line 12 takes time
$O(n^4)$ for each equation. There are $O(e)$ equations.

Finally note that in the first case of Lemma~\ref{lem:cycle:ident}, the discrimant
has only to be checked for being nonzero (which is an instance of PIT) instead of being greater than zero, since
we assume that there is always a solution.

\bibliographystyle{aaai24}
\bibliography{aaai24}

\begin{thebibliography}{27}
\providecommand{\natexlab}[1]{#1}

\bibitem[{Bollen(1989)}]{bollen2014structural}
Bollen, K.~A. 1989.
\newblock \emph{Structural equations with latent variables}.
\newblock Wiley.

\bibitem[{Bowden and Turkington(1990)}]{RePEc:cup:cbooks:9780521385824}
Bowden, R.~J.; and Turkington, D.~A. 1990.
\newblock \emph{Instrumental Variables}.
\newblock Cambridge University Press.

\bibitem[{Brito(2010)}]{brito2010instrumental}
Brito, C. 2010.
\newblock Instrumental sets.
\newblock In Dechter, R.; Geffner, H.; and Halpern, J.~Y., eds.,
  \emph{Heuristics, Probability and Causality. A Tribute to Judea Pearl},
  chapter~17, 295--308. College Publications.

\bibitem[{Brito and Pearl(2002{\natexlab{a}})}]{BritoPearlUAI02}
Brito, C.; and Pearl, J. 2002{\natexlab{a}}.
\newblock Generalized Instrumental Variables.
\newblock In \emph{Proc. Conference on Uncertainty in Artificial Intelligence
  (UAI)}, 85--93.

\bibitem[{Brito and Pearl(2002{\natexlab{b}})}]{brito2002graphical}
Brito, C.; and Pearl, J. 2002{\natexlab{b}}.
\newblock A graphical criterion for the identification of causal effects in
  linear models.
\newblock In \emph{Proc. AAAI Conference on Artificial Intelligence (AAAI)},
  533--538.

\bibitem[{Chen, Pearl, and Bareinboim(2015)}]{chen2015incorporating}
Chen, B.; Pearl, J.; and Bareinboim, E. 2015.
\newblock Incorporating knowledge into structural equation models using
  auxiliary variables.
\newblock In \emph{Proc. Int. Joint Conference on Artificial Intelligence
  (IJCAI)}, 3577--3583.

\bibitem[{Cox, Little, and {O'Shea}(1997)}]{CJO:ideals}
Cox, D.; Little, J.; and {O'Shea}, D. 1997.
\newblock \emph{Ideals, Varieties, and Algorithms}.
\newblock Springer.

\bibitem[{Drton(2018)}]{drton2018algebraic}
Drton, M. 2018.
\newblock Algebraic problems in structural equation modeling.
\newblock In \emph{The 50th anniversary of Gr{\"o}bner bases}, 35--86.
  Mathematical Society of Japan.

\bibitem[{Duncan(1975)}]{duncan2014introduction}
Duncan, O.~D. 1975.
\newblock \emph{Introduction to structural equation models}.
\newblock Academic Press.

\bibitem[{Foygel, Draisma, and Drton(2012)}]{halftrek2012}
Foygel, R.; Draisma, J.; and Drton, M. 2012.
\newblock Half-trek criterion for generic identifiability of linear structural
  equation models.
\newblock \emph{The Annals of Statistics}, 40(3): 1682--1713.

\bibitem[{Garc{\'\i}a-Puente, Spielvogel, and
  Sullivant(2010)}]{garcia2010identifying}
Garc{\'\i}a-Puente, L.~D.; Spielvogel, S.; and Sullivant, S. 2010.
\newblock {Identifying Causal Effects with Computer Algebra}.
\newblock In \emph{Proc. Conference on Uncertainty in Artificial Intelligence
  (UAI)}, 193--200.

\bibitem[{Gelfand, Kapranov, and Zelevinsky(1994)}]{resultant}
Gelfand, I.~M.; Kapranov, M.; and Zelevinsky, A. 1994.
\newblock \emph{Discriminants, resultants, and multidimensional determinants}.
\newblock Birkh\"auser.

\bibitem[{Krantz(1999)}]{moebius}
Krantz, S.~G. 1999.
\newblock \emph{Handbook of Complex Variables}, chapter Linear Fractional
  Transformations.
\newblock Birkh\"auser.

\bibitem[{Kumor, Chen, and Bareinboim(2019)}]{kumor2019instrumentalCutsets}
Kumor, D.; Chen, B.; and Bareinboim, E. 2019.
\newblock Efficient Identification in Linear Structural Causal Models with
  Instrumental Cutsets.
\newblock In \emph{Proc. Advances in Neural Information Processing Systems
  (NeurIPS)}, 12477--12486.

\bibitem[{Kumor, Cinelli, and Bareinboim(2020)}]{kumor2020auxiliaryCutsets}
Kumor, D.; Cinelli, C.; and Bareinboim, E. 2020.
\newblock Efficient identification in linear structural causal models with
  auxiliary cutsets.
\newblock In \emph{Proc. Int. Conference on Machine Learning (ICML)},
  5501--5510.

\bibitem[{Pearl(2001)}]{PearlConditionalIV}
Pearl, J. 2001.
\newblock Parameter identification: A new perspective.
\newblock Technical Report R-276, UCLA.

\bibitem[{Pearl(2009)}]{Pearl2009}
Pearl, J. 2009.
\newblock \emph{Causality}.
\newblock Cambridge University Press.

\bibitem[{Saptharishi et~al.(2021)}]{ramprasad}
Saptharishi, R.; et~al. 2021.
\newblock A selection of lower bounds in arithmetic circuit complexity.
\newblock Version 2021-07-27.

\bibitem[{Schwartz(1980)}]{pipSchwartz1980Zippel}
Schwartz, J.~T. 1980.
\newblock Fast probabilistic algorithms for verification of polynomial
  identities.
\newblock \emph{Journal of the {ACM}}, 27(4): 701--717.

\bibitem[{Stanley(2018)}]{stanley}
Stanley, R.~P. 2018.
\newblock \emph{Algebraic Combinatorics}.
\newblock Springer.

\bibitem[{van~der Zander and
  Li\'{s}kiewicz(2016)}]{instrumentalVariablesZander2016}
van~der Zander, B.; and Li\'{s}kiewicz, M. 2016.
\newblock On Searching for Generalized Instrumental Variables.
\newblock In \emph{Proc. International Conference on Artificial Intelligence
  and Statistics (AISTATS)}, 1214--1222.

\bibitem[{van~der Zander, Textor, and Li\'{s}kiewicz(2015)}]{IJCAIInstruments}
van~der Zander, B.; Textor, J.; and Li\'{s}kiewicz, M. 2015.
\newblock Efficiently Finding Conditional Instruments for Causal Inference.
\newblock In \emph{Proc. Int. Joint Conference on Artificial Intelligence
  (IJCAI)}, 3243--3249.

\bibitem[{van~der Zander et~al.(2022)van~der Zander, Wien{\"{o}}bst,
  Bl{\"{a}}ser, and Liskiewicz}]{DBLP:conf/aistats/ZanderWBL22}
van~der Zander, B.; Wien{\"{o}}bst, M.; Bl{\"{a}}ser, M.; and Liskiewicz, M.
  2022.
\newblock Identification in Tree-shaped Linear Structural Causal Models.
\newblock In Camps{-}Valls, G.; Ruiz, F. J.~R.; and Valera, I., eds.,
  \emph{International Conference on Artificial Intelligence and Statistics,
  {AISTATS} 2022, 28-30 March 2022, Virtual Event}, volume 151 of
  \emph{Proceedings of Machine Learning Research}, 6770--6792. {PMLR}.

\bibitem[{Weihs et~al.(2018)Weihs, Robinson, Dufresne, Kenkel, McGee~II,
  Reginald, Nguyen, Robeva, and Drton}]{identifyingEdgeWiseDeterminantalDrton}
Weihs, L.; Robinson, B.; Dufresne, E.; Kenkel, J.; McGee~II, K. K.~R.;
  Reginald, M.~I.; Nguyen, N.; Robeva, E.; and Drton, M. 2018.
\newblock Determinantal generalizations of instrumental variables.
\newblock \emph{Journal of Causal Inference}, 6(1).

\bibitem[{Wright(1921)}]{wright1921correlation}
Wright, S. 1921.
\newblock Correlation and causation.
\newblock \emph{J. Agricultural Research}, 20: 557--585.

\bibitem[{Wright(1934)}]{wright1934method}
Wright, S. 1934.
\newblock The method of path coefficients.
\newblock \emph{The Annals of Mathematical Statistics}, 5(3): 161--215.

\bibitem[{Zippel(1979)}]{pipZippel1979Schwartz}
Zippel, R. 1979.
\newblock Probabilistic algorithms for sparse polynomials.
\newblock \emph{Symbolic and algebraic computation}, 216--226.

\end{thebibliography}

\end{document}